\documentclass[twoside,leqno,twocolumn]{article}

\usepackage[letterpaper]{geometry}
\usepackage{makecell}
\usepackage{ltexpprt}
\usepackage{graphicx}
\usepackage{wrapfig}
\usepackage{amsmath}
\usepackage{subfig}
\usepackage{esvect}
\usepackage{url}
\usepackage{cite}
\usepackage[bold,full]{complexity}
\usepackage{color}
\usepackage{amssymb}
\usepackage{empheq} 
\usepackage{amsmath}
\usepackage{commath}
\usepackage{indentfirst}
\usepackage{amsfonts}
\usepackage{xcolor}
\usepackage{hyperref}%
\usepackage{cleveref}%
\usepackage{algorithm}
\usepackage{paralist}
\usepackage[noend]{algpseudocode}

\DeclareMathOperator*{\argmin}{\arg\!\min}

\newcommand{\fgr}[3][\relax]{%
	\begin{figure}[htp]%
		\centering
		\includegraphics[#2]{#3}%
		\ifx\relax#1\else\caption{{#1}}\fi
	\end{figure}%
}

\newcommand{\cbit}{\begin{compactitem}}
	\newcommand{\ceit}{\end{compactitem}}
\newcommand{\cben}{\begin{compactenum}}
	\newcommand{\ceen}{\end{compactenum}}

\newcommand{\bal}{\begin{align}}
\newcommand{\ean}{\end{align}}
\newcommand{\bit}{\begin{itemize}}
\newcommand{\eit}{\end{itemize}}
\newcommand{\ben}{\begin{enumerate}}
\newcommand{\een}{\end{enumerate}}
\newcommand{\beq}{\begin{equation}}
\newcommand{\eeq}{\end{equation}}

\newcommand{\mY}{\mathcal{Y}}

\newcommand{\mM}{\mathcal{M}}
\newcommand{\mF}{\mathcal{F}}

\newcommand{\mS}{\mathcal{S}}

\newcommand{\mN}{\mathcal{N}}
\newcommand{\mO}{\mathcal{O}}
\newcommand{\mT}{\mathcal{T}}
\newcommand{\mV}{\mathcal{V}}
\newcommand{\mG}{\mathcal{G}}

\newcommand{\mOk}{\mathcal{O}^{(k)}}

\newcommand{\mYk}{\mathcal{Y}^{(k)}}

\newcommand{\rhok}{\rho^{(k)}}

\newcommand{\bN}{\mathbf{N}}

\newcommand{\by}{\mathbf{y}}
\newcommand{\byk}{\mathbf{y}^{(k)}}
\newcommand{\bz}{\mathbf{z}}
\newcommand{\bzk}{\mathbf{z}^{(k)}}

\newcommand{\bu}{\mathbf{u}}
\newcommand{\bv}{\mathbf{v}}
\newcommand{\bone}{\mathbf{1}}
\newcommand{\bOk}{\mathbf{O}^{(k)}}
\newcommand{\bO}{\mathbf{O}}
\newcommand{\buk}{\mathbf{u}^{(k)}}
\newcommand{\bvk}{\mathbf{v}^{(k)}}
\newcommand{\ls}{\log^{\star}}
\newcommand{\lbdak}{\lambda^{(k)}}

\newcommand{\hide}[1]{}

\vspace{-8ex}
\date{}
\begin{document}

\title{\Large Coverage-based Outlier Explanation }


\author{Yue Wu\thanks{Department of Computer Science, University of California, Davis. yvwu@ucdavis.edu, davidson@cs.ucdavis.edu}
\and Leman Akoglu\thanks{Heinz College of Information Systems and Public Policy, Carnegie Mellon University. lakoglu@andrew.cmu.edu.}
\and Ian Davidson \footnotemark[1]}

\maketitle

\fancyfoot[R]{\scriptsize{Copyright \textcopyright\ 2020 by SIAM\\
Unauthorized reproduction of this article is prohibited}}

\begin{abstract}\small\baselineskip=9pt
\noindent Outlier detection is a core task in data mining with a plethora of algorithms that have enjoyed wide scale usage. Existing algorithms are primarily focused on detection, that is the identification of outliers in a given dataset. In this paper we explore the relatively under-studied problem of the outlier explanation problem. Our goal is, given a dataset that is already divided into outliers and normal instances, explain what characterizes the outliers. We explore the novel direction of a semantic explanation that a domain expert or policy maker is able to understand. We formulate this as an optimization problem to find explanations that are both interpretable and pure. Through experiments on real-world data sets, we quantitatively show that our method can efficiently generate better explanations compared with rule-based learners.
\end{abstract}


\section{Introduction}

Outlier detection is the task of identifying a subset of a given data set which are considered anomalous in that they are unusual from other instances~\cite{Chandola:2009:ADS:1541880.1541882}. It is one of the core data mining tasks and is central to many applications. In the security field, it can be used to identify potentially threatening users, in the manufacturing field it can be used to identify parts that are likely to fail and in social network field it can be used to identify individuals with unusual characteristics which may even be desirable. The extensive surveys on outlier detection~\cite{Chandola:2009:ADS:1541880.1541882} and graph-based outlier detection~\cite{DBLP:journals/corr/AkogluTK14} outlines many types of anomalies and associated algorithms. 

\textbf{The Need For Description versus Detection}. What all these previous methods have in common is that they only return subsets considered anomalous without \emph{describing} why those instances are anomalous. Informative descriptions that explain the behavior of outliers serve multiple purposes to multiple audiences. For data scientists, it helps to verify the correctness of their algorithms. For external parties (i.e., policy makers), it helps to understand issues such as fairness \cite{barocas2014datas, fairness_nips17}. Explanation of why something is anomalous also can be used for other purposes, such as troubleshooting. For example in the manufacturing setting it can help prevent anomalies in the future by feeding back directly into the manufacturing process \cite{manufacturing_anomaly}.

Explanation in AI is an emerging area. In the context of outlier explanation, it attempts to explain to a human why something is tagged as an outlier. Unlike outlier detection methods, the search is not for \textbf{what} is an outlier but rather \textbf{why} it is an outlier. In some cases it can be used as a natural post-processing of outlier detection algorithms. It can also be used in more general situations when outliers are \emph{already known}. Consider car manufacturing in the USA. The strong lemon laws~\cite{lemon_law} and dealership networks mean that car manufacturers already know the outliers (identified by customers and dealers) but need an explanation. 

\textbf{Method of Explanation.} Though the motivation for explanation is clear, how to explain outliers is less obvious. Explanations in terms of the objective functions of algorithms are challenging as they are not easily understood by non-data scientists. For example, graph anomalies can be detected based on statistics of the neighborhood subgraph~\cite{OddBall}. However, this uses artifacts that a domain expert or policy maker is unlikely to completely understand. Similarly, explanations in terms of the features used to find the outliers are challenging as they are often inherently complex and have little semantic meaning. 

In this work, we assume that there is a set of semantic tags for each instance which are useful for explaining as they have semantic meanings to domain experts and policy makers. The outliers could have been identified by either the same set of semantic tags or a different set of features that may be useful for detection but not for explanation. We will use the semantic tags to find meaningful descriptions to characterize the outliers. A tag can be an individual feature or conjunctions of features generated by say rule mining algorithms. We formulate the problem as a complex variation of set cover~(SC). Recall the basic set cover consists of a universe of elements and many subsets of those elements with the aim to select the smallest number of subsets to cover the entire universe. Here the notion of coverage is analogous to explanation. 

We argue that an explanation of outliers should have good interpretability and purity, that is, the description should be concise and should not cover/explain any (or just a few) normal instances. We start with the base setting with $1$ group of normal instances and $1$ group of outliers, and propose a base formulation, purity-constrained description minimization~(PDM). PDM has a sufficiency requirement, which characterizes the common behavior of outliers, and a necessity requirement, which separates the behavior of outliers from that of normal instances. We then propose two variants to the base formulation, weighted cross coverage minimization~(WCCM) and budgeted cross coverage minimization~(BCCM). In WCCM, the sufficiency and necessity requirements are relaxed so that feasible solutions are guaranteed to exist. In BCCM, we further incorporate a cardinality constraint on the description to improve the interpretability. We also propose disjoint WCCM~(D-WCCM) and disjoint BCCM~(D-BCCM) that extend WCCM and BCCM to explain multiple groups of outliers, and guarantee the orthogonality of descriptions. The proposed formulations are all intractable, so we present approximation methods for (D-)WCCM and (D-)BCCM. We compare the proposed methods with a cluster description method~\cite{DTDM}, and rule-based learners~\cite{RIPPER,BRL,BRS}, which can naturally be baselines to our coverage-based method. The experimental results show that our methods can generate interpretable and purer descriptions for outliers, and scale to large data sets.

Our contributions can be summarized as follows:
\begin{itemize}
	\item We explore the outlier explanation problem by formulating two set cover problems with constraints to produce interpretable and pure descriptions. 
	\item We analyze the complexity of our formulations and propose approximation algorithms. We prove the approximation factor for one of the algorithms.
	\item We propose metrics to evaluate purity and interpretability, and demonstrate the efficiency and effectiveness of our method by comparing with rule-based learners and a state-of-the-art cluster description method on publicly available data sets.
\end{itemize}

\section{Related Work}
\label{sec:related_work}
There are existing works on the problem of outlier/anomaly explanation~\cite{lookout, explain_by_subspace, aaai16_tom}. LookOut~\cite{lookout} finds pairs of features to pictorially explain outliers, x-PACS \cite{explain_by_subspace} discovers anomaly clusters based on feature subspaces that compactly represent anomalies and separate them from normal instances. Another work uses constrained programming to find subsets of features that maximize the difference between the number of neighbors of normal and outliers~\cite{aaai16_tom}. However, their descriptions are performed on a set of continuous features and attempt to find a sub-space as an explanation. Directly applying their work on discrete features (i.e., tags) will not produce good results. 

Our previous work on clustering is superficially similar as it finds explanations using tags. However, the descriptive clustering method~\cite{descriptive_clustering} finds clusters and explanations simultaneously. The two objectives of their work are to find dense clusters and compact descriptions, and solved via Pareto optimization. Since they use integer programming and constraint programming solvers, their method does not scale (limited to hundreds of instances).

Disjoint tag description minimization~(DTDM~\cite{DTDM}) is a cluster description method. It also returns a set of tags as an explanation of a cluster. However, it only takes the interpretability (not purity) into account because its objective is to minimize the description length with one orthogonality constraint that makes the descriptions for different clusters disjoint. If applied to our problem, DTDM does not differentiate outliers from normal instances because the tags for outliers may also apply to normal instances.

\section{Problem Definition and Formulation}
\label{sec:problem_definition_and_formulation}
We define each problem by providing the integer linear programming (ILP) formulations with data matrices as inputs. The frequently used notations are summarized as follows.
\subsection{Notations}
\cbit
    \item $\mathcal{N} = \{m_1,...,m_n\}$ denotes the set of normal instances, and $n = |\mN|$.
    \item $\mathcal{O} = \{m_1,...,m_o\}$ denotes the set of outliers, and $o = |\mO|$.
    \item $\mathcal{M} = \{m_1,...,m_{|\mathcal{M}|}\}$ denotes the set of all instances, $\mM = \mN \cup \mO$.
    \item $\mT = \{t_1,...,t_d\}$ denotes the universal set of tags (features), and $|\mT| = d$.
    \item $\bN \in \{0, 1\}^{n\times d}$ denotes the tag/data matrix of normal instances. $\bN(i, j) = 1$ denotes $m_i\in \mN$ has $t_j$.
    \item $\bO \in \{0, 1\}^{n\times d}$ denotes the tag/data matrix of outliers. $\bO(i, j) = 1$ denotes $m_i\in \mO$ has $t_j$.
    \item $\bu$ denotes the size $n$ binary selection vector for normal instances, $\bu_i = 1$ denotes $m_i\in \mN$ is ignored.
    \item $\bv$ denotes the size $o$ binary selection  vector for outliers, $\bv_i = 1$ denotes $m_i\in \mO$ is ignored.
    \item $\by$ denotes the size $d$ binary tag selection vector for outliers, $\by_j = 1$ denotes that $t_j$ is selected to explain $\mO$.
\ceit

\subsection{Sufficiency and Necessity}

Assume we are given $\mT$, $\mM$ and a tag set $T_i \subseteq \mT$ associated with each instance $m_i$, and a partition $\pi$ of $\mM$ into $2$ groups $\mN$ and $\mO$. The goal is to find a description $\mY \subseteq \mT$ for the group of outliers. The description must satisfy the sufficiency (SUFF) and necessity (NECE) described as follows.
\begin{itemize}
    \item[\emph{SUFF}] Every outlier in the group is covered by at least one tag in the description; formally, for each $m_j\in \mathcal{O}$, $\mathcal{Y} \cap T_j \ne \emptyset$.
    \item[\emph{NECE}] No normal instance exhibits any tags in the description; formally, for each $m_i\in \mathcal{N}$, $\mathcal{Y} \cap T_i = \emptyset$.
\end{itemize}
We define in the following the problem of purity-constrained description minimization~(PDM), which aims to find the most compact description that satisfies SUFF and NECE.
\begin{align}
\min_{\mathbf{y}} \quad & |\mathbf{y}| \;\;\;\; & \bullet \text{ PDM} \label{eq:oe_sn_obj}\\
\text{s.t.} \quad &  
\mathbf{O} \mathbf{y} \ge \mathbf{1} \;\;\;\; &\circ \text{ SUFF} \label{eq:oe_sn_suff}\\
& \mathbf{N} \mathbf{y} = \mathbf{0} \;\;\;\; &\circ \text{ NECE} \label{eq:oe_sn_nece}
\end{align}
PDM is our base formulation for outlier explanation. However, finding a feasible solution for PDM is very difficult in practice because the purity constraint is very strict. To resolve this issue we propose the following variants to PDM that guarantee feasible solutions while maintaining the quality of explanations.

\subsection{Relaxed Conditions}
We define the relaxed sufficiency (R-SUFF) and relaxed necessity (R-NECE) as
\begin{itemize}
    \item[\emph{R-SUFF}] The number of uncovered outliers is at most $v$; formally, $|\{T_j | T_j \cap \mY = \emptyset, m_j \in \mO\}| \le v$.
    \item[\emph{R-NECE}] The number of covered normal instances is at most $u$; formally, $|\{T_i | T_i \cap \mY \ne \emptyset, m_i \in \mN\}| \le u$
\end{itemize}
Intuitively, R-SUFF allows the solution to have uncovered outliers bounded by $v$, and R-NECE allows to have covered normal instances, which we call \emph{cross coverage}, bounded by $u$. R-NECE can guarantee that we find feasible descriptions as long as we have a large enough tolerance of the cross coverage. R-SUFF is even more meaningful in practice because outliers can have irregular patterns, or may even be misidentified. Thus, the ability to identify those uncovered ``crazy'' outliers contributes to better characterizing the behavioral pattern of outliers or re-identify outliers in an active learning setting. 

Based on these relaxations, we define in the following two variants of PDM, 1) weighted cross coverage minimization (WCCM) and 2) budgeted cross coverage minimization (BCCM). 

\subsection{Weighted Cross Coverage Minimization}
The goal of WCCM is to find $\mY \subseteq \mT$ such that the number of uncovered outliers plus the number of covered normal instances is minimized. The ILP formulation of WCCM can be written as
\begin{align}
\min_{\by, \bu, \bv} \quad & |\bu| + \lambda|\bv| \;\;\;\; & \bullet \text{ WCCM} \label{eq:obj_rsrn}\\
\text{s.t.} \quad &
\bO \by + \bv \ge \bone \;\;\;\; &\circ \text{ R-SUFF} \\
& \bN \by \le \gamma \bu, \;\;\;\; &\circ \text{ R-NECE}
\end{align}
where $\lambda$ is a weight parameter and $\gamma$ is a large constant such that $\bN_{i\cdot}\by\le \gamma$ when $\bu_i=1$. Due the nature of outliers, that is, they are greatly outnumbered by normal instances, ignoring outliers will result in a lower objective~(Eq.~\eqref{eq:obj_rsrn}). However, it is undesirable to have a description that only covers a small fraction of outliers so we add a parameter $\lambda \ge 1$ in the objective, which makes the cost of ignoring one outlier equal to ignoring $\lambda$ normal instances. Note that this formulation may not return the most compact description, i.e., smallest $|\mY|$, but we can apply SC on the results to reduce redundancy.

\noindent\textbf{Complexity Results}
We prove the following complexity results for WCCM when $\lambda=1$ by showing its relation to positive-negative partial set cover~($\pm$PSC~\cite{pnpsc}). The details of the proof can be found in the supplementary. 
\begin{theorem}
\label{thm:wccm_complexity}
    For any $\epsilon > 0$, (1) it is not possible to obtain an algorithm that returns a solution in polynomial time with approximation factor $\Omega(2^{\log^{1-\epsilon}d^4})$, unless $\NP \subseteq \DTIME(n^{\log\log n})$, and (2) it is not possible to obtain an algorithm that returns a solution in polynomial time with approximation factor $\Omega(2^{\log^{1-\epsilon}o})$, unless $\P = \NP$.
\end{theorem}

\subsection{Budgeted Cross Coverage Minimization}
The motivation of BCCM comes from 1) minimizing the cross coverage ($u$ in R-NECE), since a good description of outliers should not also apply to normal instances; and 2) making the explanation compact. One limitation of PDM is that although the objective is to minimize the length of the description, it can still be very long to cover all the outliers. However, the length of the description is a crucial measure of its interpretability, i.e., shorter descriptions are more interpretable. For instance, the decision tree can have pure leaves by growing the tree to a larger depth. However, it is no longer considered interpretable if the tree is too big, even if the description has no cross coverage between categories. Thus, we introduce a description cardinality constraint that restricts the number of selected tags. A minimum coverage constraint is then naturally incorporated so that the description can cover a desired fraction of outliers. The goal of BCCM is then to find $\mY \subseteq \mT$ such that the number of covered normal instances is minimized, and the description size and the number of uncovered outliers are bounded. We have the following ILP formulation
\begin{align}
\min_{\by, \bu, \bv}\quad & |\bu| & \bullet \text{ BCCM} \label{eq:obj_BCCM}\\
\text{s.t.} \quad &
\bO \by + \bv \ge \bone &\circ \text{ R-SUFF} \\
& \bN \by \le \gamma \bu &\circ \text{ R-NECE} \\
& |\by| \le \theta &\circ \text{ Desc. Cardinality} \label{eq:BCCM_card} \\
& |\bv| \le \rho \cdot o, &\circ \text{ Min. Coverage}
\end{align}
where $\theta > 0$ restricts the number of selected tags and $\rho \in [0, 1)$ restricts the ratio of uncovered outliers. 
$\rho=0$ indicates that all outliers must be covered.
\noindent\textbf{Complexity Results}
The following results that can be directly obtained from the hardness of weighted set cover and maximum coverage problems~\cite{sc_inapprox, max_coverage}. 
\begin{theorem}
Suppose an optimal solution has $y^{\star}$ tags that cover $u^{\star}$ normal instances and $(o - v^{\star})$ outliers, (1) it is not possible to obtain an algorithm that returns in polynomial time a collection of $((1 - \epsilon)\log o) y^{\star}$ tags that covers $(o - v^{\star})$ outliers and $\alpha u^{\star}$ normal instances, for any $\epsilon > 0$, $\alpha > 0$, unless $\NP \subseteq \DTIME(n^{\log\log n})$, and (2) it is not possible to obtain an algorithm that returns in polynomial time a collection of $y^{\star}$ tags that covers $(1 - 1/e + \epsilon)(o - v^{\star})$ outliers and $\alpha u$ normal instances, for any $\epsilon > 0$, $\alpha > 0$, unless $\NP \subseteq \DTIME(n^{\log\log n})$. 
\end{theorem}

    

\subsection{Extensions to Multiple Outlier Groups}
In practice, a partition of outliers is often given. Even if the outliers are given in one group, it is ideal to explain them in subgroups so that each subgroup can have a compact description because outliers in the same group should have more features in common. Here we discuss the extension of WCCM and BCCM into the setting where multiple groups of outliers are given. The input includes $\mT$, $\mM$, $\mN$ defined above, and $K$ groups of outliers $\mOk, k=1,\ldots,K$. Let $o^{(k)}$ be the number of outliers in $k$-th outlier group, and $o = \sum_ko^{(k)}$ denote the total number of outliers. We define the relaxed sufficiency and a relaxed necessity for multiple outlier groups as
\begin{itemize}
    \item[\emph{MR-SUFF}] For the $k$-th outlier group, the number of uncovered outliers is at most $v^{(k)}$; formally, $|\{T_j | T_j \cap \mYk = \emptyset, m_j \in \mOk\}| \le v^{(k)},\, \forall k$
    \item[\emph{MR-NECE}] For the $k$-th outlier group, the number of outliers covered by descriptions of other groups is at most $v^{(k)}$, and the number of covered normal instances is at most $u$; formally, $|\{T_j | T_j \cap \mY^{(l)} \ne \emptyset, m_j \in \mOk\}| \le v^{(k)},\,\forall k,\, l\ne k $, and $|\{T_i | T_i \cap (\bigcup_k \mYk) \ne \emptyset, m_i \in \mN\}| \le u$.
\end{itemize}
Given instance-tag matrices $\mathbf{N}$ and $\bOk$, $k=1..K$, we are solving for $\{\by^{1},...,\by^{(K)}\}$ and binary vectors $\mathbf{u} \in \{0, 1\}^n$, $\bvk \in \{0, 1\}^{o^{(k)}}, \forall k$. The multi-group extension of WCCM can be defined in the following ILP formulation 
\begin{align}
\min_{\{\byk\}, \bu, \{\bvk\}} & |\bu| + \sum_k\lbdak |\bvk| \label{eq:obj_D-WCCM} \;\;\;\; \bullet \text {D-WCCM} \\
\text{s.t.} \quad &
\bOk \byk + \bvk \ge \bone, \, \forall k \label{eq:mrsuff_ilp}\\
& \mathbf{N} \byk \le \gamma \bu,\, \forall k  \label{eq:mrnece_n_ilp}\\
& \bOk \by^{(l)} \le \gamma \bvk,\, \forall k,\, l\ne k \label{eq:mrnece_o_ilp}
\end{align}
It is worthwhile to mention that this formulation implicitly guarantees that the descriptions found for each outlier group are disjoint, i.e., $\sum_k \byk \le \bone$. Thus, we call this extension disjoint weighted cross coverage minimization~(D-WCCM). 

For the multi-group extension of BCCM (D-BCCM), we add cardinality and minimum coverage constraints on each group. We have
\begin{align}
    \min_{\{\byk\}, \bu, \{\bvk\}} & |\bu| & \bullet \text{ D-BCCM} \label{eq:dbccm_obj}\\
    & |\byk| \le \theta^{(k)},\, \forall k \label{eq:desc_card_D-BCCM_ilp}\\
    & |\bvk| \le \rhok o^{(k)},\, \forall k \label{eq:min_cover_D-BCCM_ilp},
\end{align}
with additional constraints Eq.~\eqref{eq:mrsuff_ilp} (MR-SUFF), Eq.~\eqref{eq:mrnece_n_ilp} and Eq.~\eqref{eq:mrnece_o_ilp}(MR-NECE).

\section{Approximations Methods}
\label{sec:approximation}
We focus on the approximation algorithm for D-WCCM and D-BCCM as they also apply to 1 outlier group with few simplifications. We propose two algorithms, Greedy-D-WCCM and Greedy-D-BCCM that are based on the heuristic of the greedy set cover algorithm, and we prove the approximation factor for D-WCCM. Greedy-D-BCCM does not have a performance guarantee, but we demonstrate that it performs very well in practice.

\subsection{Approximation to D-WCCM} 
\begin{algorithm}[h]
	\hspace*{\algorithmicindent} \textbf{Input} $\mN$, $\{\mOk\}$, $\mT$, $\{\lbdak\}$
	\caption{Greedy-D-WCCM}\label{alg:D-WCCM}
	{\fontsize{10}{11}\selectfont
	\begin{algorithmic}[1]
		\State $\mO = \bigcup_k\mOk$ \text{\quad // universal set to cover}
		\State $\mS = \emptyset$ \text{\quad // collections of sets}
		\State $\mYk = \emptyset, \, \forall k$ \text{\quad // selected tags}
		\State $\mG = \{1, \ldots, K\}$ \text{\quad // groups to cover}
		\For{$j = 1$ \textbf{to} $d$} \label{alg:D-WCCM_create_set_and_weight_start}
		\State $S_j=\{m_i\in \mO | t_j \in T_i \}$
		\State $\mS = \mS \cup S_j$
		\For{$k = 1$ \textbf{to} $K$} \label{alg:D-WCCM_weight_assign_begin}
		\State calculate $w(S_j, k)$ using Eq.~\eqref{eq:D-WCCM_weight}\label{alg:D-WCCM_weight_assign_end} \label{alg:D-WCCM_create_set_and_weight_end}
		\EndFor
		\EndFor
		\State $ind = sort(\{\lbdak\})$ \label{alg:D-WCCM_lambda_sort}
		\For{$k = 1$ \textbf{to} $K$} \label{alg:D-WCCM_wsc_begin}
		\State $\tau=\lambda^{(ind[k])}/(1 + \epsilon)$ \label{alg:D-WCCM_price_of_ignoring}
		\While{$\mO \ne \emptyset$}
		\State $j', k' = \argmin_{S_j \in \mS, k \in \mG}e(S_i, j)$ (Eq.~\eqref{eq:D-WCCM_price}) \label{alg:D-WCCM_select_min_price} 
		\If{$e(S_{j'}, k') \le \tau$}
		\State $\mY^{(k')} = \mY^{(k')} \cup t_{j'}$
		\State $\mO = \mO \setminus S_{j'}$
		\Else 
		\State $\mG = \mG \setminus \{ind[k]\}$
		\State \textbf{break}
		\EndIf
        \EndWhile
        \EndFor \label{alg:D-WCCM_wsc_end}
		\State \textbf{return} $\{\mYk\}$ 
	\end{algorithmic}}
\end{algorithm}
\setlength{\textfloatsep}{8pt}

Define the cost of a set $S$ w.r.t. the $k$-th outlier group as
\begin{equation}
w(S, k) = \sum_{l\ne k}\lbdak|\mO^{(l)} \cap S| + |\mN \cap S|, \label{eq:D-WCCM_weight}
\end{equation}
where $\lbdak$ is the weight assigned to ignoring an outlier in group $k$. Also define the price of picking set $S$ w.r.t. the $k$-th outlier group as \begin{equation}
e(S, k) = \frac{w(S, k)}{|S \cap \mO| + \epsilon}.\label{eq:D-WCCM_price}
\end{equation}
We have an approximation algorithm described in Algorithm~\ref{alg:D-WCCM}. The collection of weighted sets is constructed in lines~\ref{alg:D-WCCM_create_set_and_weight_start}-\ref{alg:D-WCCM_create_set_and_weight_end}. The elements in set $S_i$ are all outliers using tag $t_i$, regardless of which group they belong to. Each set has a weight calculated using Eq.~\eqref{eq:D-WCCM_weight}. The weight of a set if used to cover group $k$ is the total number of instances in groups $l\ne k$ covered by that set, multiplied by the corresponding group weights. Line~\ref{alg:D-WCCM_lambda_sort} sorts $\{\lbdak\}$ in ascending order and stores the indices in array $ind$. Lines~\ref{alg:D-WCCM_wsc_begin}-\ref{alg:D-WCCM_wsc_end} depict the process of weighted set cover. Note that the greedy algorithm always picks the set with the lowest price. If the price of a set to cover group $k$ is larger than the price of ignoring the element in group $k$, the algorithm will ignore the elements. That suggests when the lowest price (of any set to cover any group) exceeds the price of ignoring the element in group $k$, all uncovered outliers in group $k$ will never be covered. So we calculate a threshold price $\tau$ in line~\ref{alg:D-WCCM_price_of_ignoring}. Line~\ref{alg:D-WCCM_select_min_price} returns the indices of the selected set and group. We add the selected set to the solution and update $\mO$ if the price is not larger than the current $\tau$. Otherwise, we do not cover elements in group $ind[k]$ in future iterations. 

\textbf{Analysis} We prove the approximation factor for Algorithm~\ref{alg:D-WCCM}. Define $\mV$ to be the set that includes outliers not covered by $\mY$ and normal instances covered by $\mY$. Also define $\deg(m, \mT)$ to be the number of tags in $\mT$ that exhibits instance $m$, $w(\mY)$ to be the sum of weights by selecting sets in line~\ref{alg:D-WCCM_select_min_price}. Also define the per-instance cost 
$
    w(m) =
  \begin{cases}
   \lbdak & \text{if $m\in \mOk,\, \forall k$} \\ 
    1 & \text{if $m\in \mN$} \\
  \end{cases}. \nonumber
$

\begin{lemma}
	$w(\mY) \le d\cdot \sum_{m_i\in \mV} w(m_i).$
\end{lemma}
\begin{proof}
	According to Eq.~\eqref{eq:D-WCCM_weight}, $w(\mY)$ is an upper bound of true cost. We show the bound by the fact that each element $m_i$ is counted $\deg (m_i, \mY)$ times,
	\begin{equation}
	\begin{split}
	w(\mY) & = \sum_{m_i \in \mV}\deg(m_i, \mY)\cdot w(m_i) \\
	& \le \max_{m_i \in \mV} \left(\deg(m_i, \mY)\right) \cdot \sum_{m_i \in \mV}w(m_i) \\
	& \le |\mY| \cdot \sum_{m_i \in \mV}w(m_i) \\
	& \le d \cdot \sum_{m_i \in \mV}w(m_i), 
	\end{split}
	\end{equation}
	implying the lemma.
\end{proof}
\begin{theorem}
	Algorithm~\ref{alg:D-WCCM} gives an approximation factor of $d \cdot \ln{o}$.
\end{theorem}
\begin{proof}
	Algorithm~\ref{alg:D-WCCM} can be seen as a variant of the weighted set cover. Line~\ref{alg:D-WCCM_weight_assign_begin}-\ref{alg:D-WCCM_weight_assign_end} assigns $K$ weights to each set for determining which outlier group to cover. This is equivalent to making $K$ copies of the same set and assigning different weights to each copy. Once one of the $K$ sets is selected, the rest $K-1$ sets will never be selected again based on Eq.~\eqref{eq:D-WCCM_price} as they do not cover any elements in $\mO$. Since the greedy weighted set cover gives an approximation factor of $\ln{o}$~\cite{sc_greedy}, we have
	\begin{equation}
	\begin{split}
	w(\mY) & \le d\cdot \sum_{m_i \in \mV}w(m_i) \\
	& \le d\cdot \ln{o} \cdot w(\mY^{\star}),
	\end{split}
	\end{equation}
	where $w(\mY^{\star})$ denotes the optimal cost of D-WCCM. Thus, 
	Algorithm~\ref{alg:D-WCCM} gives an approximation factor of $d \cdot \ln{o}$.
\end{proof}
In practice, $\max_{m_i \in \mV} (\deg(m_i, \mY))$ is often much smaller the $d$, and $o$ is usually small due the definition of outliers. We also notice that we can improve the approximation factor by iteratively guessing the largest weight of a set in the optimal solution~\cite{rbsc_approx}. However, the running time of that algorithm grows linearly with the number of normal instances so it  cannot be used in practice. 

\subsection{Approximation to D-BCCM} 

\begin{algorithm}[h]
	\hspace*{\algorithmicindent} \textbf{Input} $\mN$, $\{\mOk\}$, $\mT$, $\{\rhok\}$, $\{\theta^{(k)}\}$
	\caption{Greedy-D-BCCM}\label{alg:D-BCCM}
	{\fontsize{10}{11}\selectfont
	\begin{algorithmic}[1]
		\State $\mS^{(k)} = \emptyset,\, \forall k$
		\State $\mY^{(k)} = \emptyset,\, \forall k$
		\State $\theta^{(k)}_r = \theta^{(k)},\, \forall k$
		\For{$k = 1 \textbf{ to } K$}\label{alg:D-BCCM_create_set_and_weight_start}
		\State $r^{(k)} = (1 - \rhok)o^{(k)}$
		\State $\overline{r}^{(k)} = o^{(k)} - r^{(k)}$
		\For{$j = 1 \textbf{ to } d$}
		\State $S_j^{(k)}=\{m_i\in \mOk | t_j \in T_i \}$ \label{alg:D-BCCM_set_create}
		\State $w(S_j^{(k)}) = |\{m_i\in \mN | t_j \in T_i \}|$
		\State $\mS^{(k)} = \mS^{(k)} \cup S_j^{(k)}$
		\EndFor
		\EndFor
		\label{alg:D-BCCM_create_set_and_weight_end}
		\For{$k = 1 \textbf{ to } K$}
		\For{$j = \theta \textbf{ downto } 1$}
		\State $\alpha^{(l)} = \frac{\overline{r}^{(k)}}{\sum_{l\ne k}\theta^{(l)}_r},\, \forall l\ne k$ \label{alg:D-BCCM_heuristic_alpha}
	    \State $\beta = \frac{r^{(k)}}{j}$ \label{alg:D-BCCM_heuristic_beta}
		\State \makecell[l]{select $S^{(k)}_{j'}$ such that $|S^{(k)}_{j'} \cap \mOk| \ge \beta$,\\ $|S^{(k)}_{j'} \cap \mO^{(l)}| \le \alpha^{(l)}, \forall l\ne k$, and \\$e(S_{j'}^{(k)}, \mOk, w(S_{j'}^{(k)}))$ is minimized} \label{alg:D-BCCM_select_set}
		\If{ $j == NULL$} 
		\State \textbf{return} $\emptyset$ 
		\EndIf
		\State $\mYk = \mYk \cup t_{j'}$ \label{alg:D-BCCM_update_start}
		\State $r^{(k)} = r^{(k)} - |\mOk \cap S^{(k)}_{j'}|$
		\State $\overline{r}^{(l)} = \overline{r}^{(l)} - |\mO^{(l)} \cap S^{(l)}_{j'}|,\, \forall l \ne k$
		\State $\mO^{(l)} = \mO^{(l)} \setminus S^{(l)}_{j'},\, \forall l=1,\ldots,K$ \label{alg:D-BCCM_update_uncovered_set}
		\State $\theta^{(k)}_r = \theta^{(k)}_r - 1$ \label{alg:D-BCCM_update_end}
		\If{$r^{(k)} \le 0$} \label{alg:D-BCCM_check_terminate_start}
		\State $\theta^{(k)}_r = 0$
		\State \textbf{break}
		\Else
		\If{$i == 1$}
		\State \textbf{return} $\emptyset$ 
		\EndIf
		\EndIf \label{alg:D-BCCM_check_terminate_end}
		\EndFor
		\EndFor
		\State \textbf{return} $\{\mYk\}$ 
	\end{algorithmic}}
\end{algorithm}

The greedy algorithm of D-BCCM is described in Algorithm~\ref{alg:D-BCCM}. The collection of sets and their weights are constructed in lines~\ref{alg:D-BCCM_create_set_and_weight_start}-\ref{alg:D-BCCM_create_set_and_weight_end}. $r^{(k)}$ and $\overline{r}^{(k)}$ will be used later for the heuristics. $r^{(k)}$ denotes the remaining number of elements in group $k$ to cover, and $\overline{r}^{(k)}$ denotes the number of elements in group $k$ that are allowed to be covered by descriptions of other groups. We create a set for each group $k$ and tag $i$ in line~\ref{alg:D-BCCM_set_create} by adding elements in group $k$ that use tag $i$. The weight is equal to the number of normal instances covered by tag $i$. 
This algorithm covers elements group by group based on two heuristics calculated in lines~\ref{alg:D-BCCM_heuristic_alpha} and \ref{alg:D-BCCM_heuristic_beta}. The intuition is to always pick the set with sufficient coverage and it does not cover too many elements in other groups so that their minimum coverage constraints would not be violated. Thus, in line~\ref{alg:D-BCCM_select_set} we select the set that has minimum price $e$ and also covers at least $\beta$ elements in the current group $k$, at most $\alpha^{(l)}$ elements in other groups $l\ne k$. $\alpha$ and $\beta$ are calculated by the remaining number of elements to cover (or ignore) divided by the remaining number of rounds. This strategy guarantees that the minimum coverage constraint of each group can be satisfied under the budget constraint. Lines~\ref{alg:D-BCCM_update_start}-\ref{alg:D-BCCM_update_end} update the solution set, the remaining number of elements to cover, the allowed number of ignored elements, and the sets of uncovered elements. Note that line~\ref{alg:D-BCCM_update_uncovered_set} indicates that once elements are covered by the description of other groups, they are removed from the set and thus will not be covered by the description of their group. It guarantees the orthogonality of descriptions as their coverages do not overlap. Lines~\ref{alg:D-BCCM_check_terminate_start}-\ref{alg:D-BCCM_check_terminate_end} check if the minimum coverage has been satisfied. We break the loop if it is satisfied. Otherwise, we continue to find the next tag if we are still within budget.

If there is a single outlier group, $\alpha$ is no longer needed. We only need to ensure that at each iteration the selected tag covers at least $\beta$ elements in line~\ref{alg:D-BCCM_heuristic_beta}.

\section{Experiments}
In the experiments, we focus on the performance of Greedy-D-WCCM and Greedy-D-BCCM for explaining multiple groups of outliers.
\vskip -0.1in
\begin{table}[h]
    \centering
    \caption{Summary of the data sets. The number of outliers by group are given in parentheses.}
    \resizebox{0.9\linewidth}{!}{%
    \begin{tabular}{c|ccc}
    \Xhline{2\arrayrulewidth}
        data set & \# instances & \# outliers & \# features\\
        \hline
        {\sc Readmission} & $54,864$ & $2,887$ $(2,261/626)$ & $55$ \\
        {\sc Census} & $26,020$ & $1,301$ $(553/369/379)$ & $14$ \\
        {\sc Heart Disease} & $303$ & $139$ $(55/36/35/13)$ & $14$ \\
        \Xhline{2\arrayrulewidth}
    \end{tabular}}
    \label{tab:data_statistics}
\end{table}
\vskip -0.3in
\label{sec:experiment}
\subsection{Data}
We use three data sets, {\sc Census}~\cite{census}, {\sc Readmission}~\cite{readmission}, and {\sc Heart Disease}~\cite{heart_disease} from UCI Machine Learning Repository. These data sets have mostly discrete features and thus are ideal for the evaluation. Other continuous features are discretized. Classes with few instances are naturally treated as outlier groups. We further downsample the outlier group (except for the {\sc Heart Disease}) so that outliers take up approximately $5\%$ of each data set. The statistics of the processed data sets are summarized in Table~\ref{tab:data_statistics}. We divide the outliers in {\sc Readmission} and {\sc Census} into two and three groups, respectively by k-means clustering. {\sc Heart Disease} has one major and four minor classes, which are used as outlier groups. These account for two settings in reality, explaining outliers in subgroups and explaining outliers with a given partition. To create the universal tag set $\mT$ as the input for our methods, we use random forest to generate $2,000$ rules with maximum length of three. 
\subsection{Metrics}
Different metrics are used to evaluate the purity and interpretability of our methods. The purity metrics include 
\cbit
    \item \textbf{Fraction of uncovered outliers~(FUO)} This metric calculates the fraction of outliers that are not covered by any of the selected tags.
    \item \textbf{Fraction of covered outliers~(FCO)} This metric is only used in the setting of explaining multiple groups of outliers. It calculates the fraction of outliers not in the target group covered by the description of that group.
    \item \textbf{Fraction of Covered normal instances~(FCN)} This metric calculates the fraction of normal instances that are covered by at least one of the selected tags.
\ceit
The interpretability metrics include 
\cbit
    \item \textbf{Average length~(AL)} This metric calculates the average number of predicates in each selected tag.
    \item \textbf{Number of tags~(NT)} This metric calculates the number of selected tags.
\ceit
Since the purity and interpretability are incompatible in general, we then propose to use minimum description length~(MDL) as a \emph{unified} metric which returns a single score for each method. The MDL metric calculates the cost (number of bits) to transfer the results returned by the algorithm from Sender to Receiver. Define $\mF$ to be the $(n+o)\times \Tilde{d}$ binary data matrix where $\Tilde{d}$ is the number of unique binary features (e.g., gender:male). Sending the  sizes $|\mN|, \{|\mOk|\}, |\mF|$ and $\mF$ is fixed cost for all explanations. Let us assume Receiver has the knowledge of those elements upon Sender communicating them first. Define 
$N_k = |\mOk| - |\bvk| + |\buk| + |\bzk|$ to be the number of instances covered by $\byk$. 
Then, Sender aims to communicate $\{\mOk\}$, i.e. the specific instances in each outlier group, as follows.

For each outlier group $k$, encode
\cbit
    \item \textit{unique binary features in description} (for \textbf{NT} and \textbf{AL}): $\byk$ is the binary vector indicating the tags in the description for group $k$. Each tag consists of conjunctions of unique binary features. Let $n_t$ denote the length (number of predicates) of a tag. The number of bits can be calculated as $\sum_{t\in \byk}  \left(\ls n_t + \log_2 { \binom{\Tilde{d}}{n_t}} + n_t\right)$. First we send the number of conjunctions per tag, followed by their identities/indices, and lastly their values (0 or 1, i.e., 1 bit per conjunction). 
    \item \textit{the unexplained (`crazy') outliers} (for \textbf{FUO}): $\bvk$ is the  binary vector indicating outliers in group $k$ not covered by $\byk$. We need $\left(\ls |\bvk| + \log_2 {\binom{n+o-N_k}{|\bvk|}}\right)$ bits.
    \item \textit{the explained normal instances} (for \textbf{FCN}): $\buk$ is the binary vector indicating normal instances covered by $\byk$. We need $\left(\ls |\buk| + \log_2 { \binom{N_k}{|\buk|}}\right)$ bits.
     \item \textit{the explained outliers from other groups} (for \textbf{FCO}): $\bzk$ is the  binary vector indicating outliers not in group $k$ covered by $\byk$. We need $\left(\ls |\bzk| + \log_2 { \binom{N_k}{|\bzk|}}\right)$ bits.
\ceit
Our algorithm only returns $\{\byk\}$. Sender can obtain $\{\buk\}$, $\{\bvk\}$, and $\{\bzk\}$ given their knowledge of the data (i.e., $\mN,\,\{\mOk\}$) as well as $\{\byk\}$. When $k=1$, the terms involving $\bz$ are dropped. Overall, the MDL cost can be calculated as follows
\begin{align*}
& \text{MDL}(\{\buk\}, \{\bvk\}, \{\byk\}, \{\bzk\}) \\
= & \sum_k \Bigg(\sum_{t\in \byk}  \left(\ls n_t + \log_2 { \binom{\Tilde{d}}{n_t}} + n_t\right) + \\
& \ls |\bvk| + \log_2 {\binom{n+o-N_k}{|\bvk|}} + \ls |\buk| + \\
& \log_2 { \binom{N_k}{|\buk|}} + \ls |\bzk| + \log_2 { \binom{N_k}{|\bzk|}} 
\Bigg)
\end{align*}
Note that the larger the $\byk$, $n_t$ for each $t\in \byk$, $|\bvk|$, $|\buk|$ and 
$|\bzk|$, the larger the MDL cost and the worse (i.e., lengthier) the explanation.

Lower is better for all proposed metrics.
\begin{table*}[t]
    \centering
    \caption{Comparisons on {\sc Readmission} with 2 outlier groups. The score obtained on each outlier group is separated by ``$/$''. We group the metrics into three categories, purity~(FUO, FCO, FCN), interpretability~(AL, NT) and overall measures~(MDL, running time). Lower is better for all metrics. The lowest MDL is in bold.}
    \resizebox{0.75\linewidth}{!}{%
    \begin{tabular}{cccc|cc|cc}
    \Xhline{2\arrayrulewidth}
      & \textbf{FUO}  & \textbf{FCO}& \textbf{FCN} & \textbf{AL} & \textbf{NT} & \textbf{MDL} & \textbf{Time (s)} \\
      \Xhline{2\arrayrulewidth}
      RIPPER & $1.00 / 0.06$ & $1.00 / 0.95$ & $0.26/0.95$ & $11.00 / 1.00$ & $1 / 2$ & $35,420$ & $1,592$\\
      BRL & $0.85 / 0.00$ & $0.14 / 1.00$ & $0.11 / 0.97$ & $1.00 / 1.00$ & $1 / 1$ & $32,097$ & $7,355$ \\
      BRS & $1.00 / 0.00$ & $0.00 / 1.00$ & $0.74 / 0.24$ & $3.00 / 3.00$ & $2 / 2$ & $24,174$ & $411$ \\
      DTDM & $0.00 / 0.00$  & $0.00 / 1.00$ & $0.75 / 1.00$ & $1 / 1$ & $1 / 1$ & $23,081$ &$0.23$\\
      \hline
      \textbf{(Ours) D-WCCM} & $0.04/0.12$ & $0.00/0.00$ & $0.69/0.22$ & $2.94/3.00$ & $119/18$ & $10,950$ & $13.15$\\
      \textbf{(Ours) D-BCCM} & $0.05/0.10$ & $0.00/0.00$ & $0.68/0.21$ & $2.75/ 2.75$ & $4/ 4$ & $\mathbf{6,800}$ & $0.28$\\
      \Xhline{2\arrayrulewidth}
    \end{tabular}}
    \label{tab:readmission}
\end{table*}
\begin{table*}[h]
    \centering
    \caption{Comparisons on {\sc Census} with $3$ outlier groups. For each purity and interpretability metric, three scores obtained on the three outliers groups are presented for each method.}
    \resizebox{0.9\linewidth}{!}{%
    \begin{tabular}{cccc|cc|cc}
    \Xhline{2\arrayrulewidth}
      & \textbf{FUO}  & \textbf{FCO}& \textbf{FCN} & \textbf{AL} & \textbf{NT} & \textbf{MDL} & \textbf{Time (s)} \\
      \Xhline{2\arrayrulewidth}
      RIPPER & $0.99 / 0.97 / 0.00$ & $0.97 / 0.73 / 1.00$ & $0.55 / 0.54 / 1.00$ & $3.50 / 1.00 / 1.00$ & $2 / 1 / 3$ & $21,928$ &$735$\\
      BRL & $0.95 / 1.00 / 0.00$ & $0.03 / 0.00 / 0.66$ & $0.03 / 0.00 / 0.75$ & $1.00 / 3.00 / 1.00$ & $1 / 1 / 1$ & $12,082$ & $1,716$ \\
      BRS & $0.10 / 1.00 / 0.00$ & $0.48 / 0.00 / 0.00$ & $0.40 / 0.29 / 0.26$ & $3.00 / 3.00 / 3.00$ & $3 / 1 / 2$ & $7,795$ & $359$ \\
      DTDM & $0.00 / 0.00 / 0.00$  & $0.99 / 0.94 / 0.00$ & $1.00 / 0.94 / 0.27$ & $1.00 / 1.00 / 1.00$ & $1 / 1 / 1$ & $15,603$ & $4.44$\\
      \hline
      \textbf{(Ours) D-WCCM} & $0.4 / 0.37 / 0.02$ & $0.03 / 0.02 / 0.00$ & $0.43 / 0.18 / 0.26$ & $2.91 / 2.878 / 2.94$ & $73 / 38 / 29$ & $8,075$ & $7.56$\\
      \textbf{(Ours) D-BCCM} & $0.13 / 0.12 / 0.01$ & $0.02 / 0.13 / 0.00$ & $0.38 / 0.34 / 0.26$ & $3.00 / 2.33 / 2.67$ & $5 / 3 / 3$ & $\mathbf{4,907}$ & $0.08$ \\
      \Xhline{2\arrayrulewidth}
    \end{tabular}}
    \label{tab:census}
\end{table*}

\subsection{Baselines}
We benchmark the performance of our methods against the following rule-based models for classification: Bayesian Rule Lists~\cite{BRL}, Repeated Incremental Pruning to Produce Error Reduction (RIPPER)~\cite{RIPPER}, and Bayesian Rule Sets (BRS)~\cite{BRS}. We treat all samples as training instances and apply the rule-based learner to generate descriptions in disjunctive normal form. For multiple groups of outliers, we adopt the one-vs.-rest strategy to generate descriptions since they are proposed for binary classification. We post-process the output of BRL by selecting a subset of rules with $>50\%$ or the one with the highest confidence score from its decision list. Another baseline is DTDM~\cite{DTDM}, which is also a coverage-based method for cluster description. 

\subsection{Results}
Through the experimental results, we aim to demonstrate that 1) the descriptions of our methods have better overall performance measured by MDL; 2) our methods can generate purer descriptions than baselines with comparable interpretability; 3) our methods are more versatile to produce descriptions with different trade-offs between purity and interpretability.

We first focus on the unified MDL metric. For each method, we experiment with various parameter settings and report the results with the lowest MDL in Table~\ref{tab:readmission}, \ref{tab:census} and \ref{tab:heart_disease}. The details on the parameter setting can be found in the supplementary. BRL does not return any rule set for groups $1$ and $4$ on {\sc Heart Disease} so we do not calculate its MDL. The lowest MDL is in bold. We also include the running time in the table to demonstrate the efficiency of our methods. Our D-BCCM has the lowest MDL on all three data sets. D-WCCM has the second lowest MDL on {\sc Readmission} and third lowest on {\sc Census} and {\sc Heart Disease}. These results demonstrate the strong overall performance of our methods. 

Table~\ref{tab:readmission}, \ref{tab:census} and \ref{tab:heart_disease} also show that RIPPER, BRL, and DTDM cannot differentiate a group of outliers from other instances. On {\sc Readmission}, RIPPER ignores all outliers in the first group. Although it covers most of the outliers in the second group, the description of the second group covers almost all other instances. This observation applies to other data sets. BRL performs similarly to RIPPER. DTDM does not have purity constraints so it covers many normal instances. Although they can be considered interpretable since their descriptions consist of very few tags, their explanations are not meaningful. On the other hand, these results demonstrate that our D-WCCM and D-BCCM can produce purer descriptions since all of FUO, FCO and FCN are relatively low, and the interpretability of D-BCCM is comparable to baselines.
\begin{table}[h]
    \centering
    \caption{We use various settings of parameters to investigate how $\lambda$ of D-WCCM and $\rho, \theta$ of D-BCCM can influence the trade-off between metrics on {\sc Readmission} with two outlier groups. We set the same $\rho$ and $\theta$ of D-BCCM for each group. }
    \resizebox{\linewidth}{!}{%
    \begin{tabular}{cccc|cc}
    \Xhline{2\arrayrulewidth}
      & \textbf{FUO}  & \textbf{FCO}& \textbf{FCN} & \textbf{AL} & \textbf{NT} \\
      \Xhline{2\arrayrulewidth}
      $\lambda=\{15, 15\}$ & $0.69/0.85$ & $0.00/0.00$ & $0.12/0.02$ & $2.88/2.93$ & $107/13$ \\
      $\lambda=\{15, 25\}$ & $0.69/0.61$ & $0.00/0.00$ & $0.11/0.07$ & $2.89/2.94$ & $103/29$ \\
      $\lambda=\{30, 50\}$ & $0.31/0.18$ & $0.00/0.00$ & $0.40/0.19$ & $2.88/2.98$ & $143/37$\\
      \hline
      $\rho, \theta=0.2, 5$ & $0.11/0.14$ & $0.00/0.00$ & $0.63/0.20$ & $2.60/ 2.50$ & $5 / 2$\\
      $\rho, \theta=0.2, 10$ & $0.20/0.10$ & $0.00/0.00$ & $0.54/0.21$ & $2.90/ 2.75$ & $10 / 4$\\
      $\rho, \theta=0.4, 10$ & $0.38/0.32$ & $0.05/0.00$ & $0.38/0.14$ & $2.90/ 3.00$ & $10 / 3$\\
      $\rho, \theta=0.8, 10$ & $0.79/0.79$ & $0.00/0.00$ & $0.09/0.04$ & $2.89/3.00$ & $9 / 4$\\
      \Xhline{2\arrayrulewidth}
    \end{tabular}}
    \label{tab:parameter_tuning}
\end{table}

\begin{table*}[t]
    \centering
    \caption{Comparisons on {\sc Heart Disease} with $4$ outlier groups. The convention follows Table~\ref{tab:readmission} and \ref{tab:census}.}
    \resizebox{0.95\linewidth}{!}{%
    \begin{tabular}{cccc|cc|cc}
    \Xhline{2\arrayrulewidth}
      & \textbf{FUO}  & \textbf{FCO}& \textbf{FCN} & \textbf{AL} & \textbf{NT} & \textbf{MDL} & \textbf{Time (s)} \\
      \Xhline{2\arrayrulewidth}
      RIPPER & $0.84 / 1.00 / 0.97 / 0.77$ & $0.19 / 0.14 / 0.30 / 0.60$ & $0.44 / 0.59 / 0.88 / 0.93$ & $1.00 / 2.00 / 2.00 / 1.00$ & $1 / 1 / 2 / 2$ & $1,169$ & $104.82$\\
      BRL & $1.00 / 0.19 / 0.60 / 1.00$ & $0.00 / 0.71 / 0.16 / 0.00$ & $0.00 / 0.21 / 0.04 / 0.00$ & $0.00 / 1.00 / 1.00 / 0.00$ & $0 / 1 / 1 / 0$ & - & $57.02$ \\
      BRS & $0.91 / 0.86 / 0.71 / 0.77$ & $0.01 / 0.00 / 0.00 / 0.00$ & $0.00 / 0.00 / 0.00 / 0.00$ & $3.00 / 3.00 / 3.00 / 3.00$ & $1 / 2 / 3 / 2$ & $656$ & $43.06$ \\
      DTDM & $0.00 / 0.00 / 0.00 / 0.00$  & $0.37 / 0.34 / 0.33 / 0.39$ & $1.00 / 0.40 / 0.49 / 0.68$ & $1.00 / 1.00 / 1.00 / 1.00$ & $2 / 2 / 2 / 2$ & $3,610$ & $18.13$\\
      \hline
      \textbf{(Ours) D-WCCM} & $0.69 / 0.36 / 0.37 / 0.85$ & $0.15 / 0.09 / 0.12 / 0.01$ & $0.01 / 0.02 / 0.01 / 0.00$ & $3.00 / 3.00 / 3.00 / 3.00$ & $1 / 5 / 6 / 6$ & $789$ & $0.53$\\
      \textbf{(Ours) D-BCCM} & $0.93 / 0.83 / 0.63 / 0.69$ & $0.02 / 0.08 / 0.15 / 0.02$ & $0.0 / 0.00 / 0.00 / 0.00$ & $3.00 / 3.00 / 3.00 / 3.00$ & $2 / 1 / 1 / 1$ & $\mathbf{642}$ & $0.02$\\
      \Xhline{2\arrayrulewidth}
    \end{tabular}}
    \label{tab:heart_disease}
\end{table*}
We then investigate how the parameters can affect the descriptions of our methods. The results on {\sc Readmission} are presented in Table~\ref{tab:parameter_tuning}, more results can be found in the supplementary. We find that $\lambda$ of D-WCCM can influence FUO, FCO and FCN. When both groups have $\lambda=15$, the second group with fewer outliers has higher FUO. Then, FUO of the second group decreases with higher $\lambda$. When we increase $\lambda$ for both groups, their FUOs are both dropped. We also observe that lower FUO inevitably increases FCN. We provide four sets of $\rho$ and $\theta$ for D-BCCM to show how they individually affect FUO, FCN and NT. We first fix $\rho=0.2$ and set $\theta=\{5, 10\}$. From the first two sets of results, we find that restricting NT will increase FCN. It indicates that the interpretability is not compatible with purity. We then fix $\theta=10$ and set $\rho=\{0.2,0.4,0.8\}$. We find that increasing $\rho$ will lower FCN, which further validate that the FCN is incompatible with FUO. 

Although the purity and interpretability metrics are not all compatible with each other, the results in Table~\ref{tab:parameter_tuning} demonstrate that our methods can easily trade off among these metrics, thus, can generate useful descriptions according to the user's needs (since our methods are also efficient). 

\section{Conclusion}
\label{sec:conclusion}
In this work, we formulate the problem of outlier explanation as coverage problems with constraints. We propose two formulations, (D-)WCCM and (D-)BCCM to produce high-quality explanations. We analyze the complexity of our formulations, and propose corresponding approximation algorithms Greedy-D-WCCM and Greedy-D-BCCM. The experiment results on three data sets demonstrate that the proposed algorithms have better overall performance measured by MDL. Our methods can also produce significantly purer descriptions while the interpretability of descriptions is comparable to baselines. Experiments with different parameter settings further show that our algorithms can also adjust descriptions with different trade-offs between purity and interpretability. 

\bibliographystyle{siamplain} 
\bibliography{references}

\begin{thebibliography}{10}

\bibitem{max_coverage}
{\em Approximation Algorithms for NP-hard Problems}, PWS Publishing Co., 1997.

\bibitem{lemon_law}
{\sc G.~A. Akerlof}, {\em The market for "lemons": Quality uncertainty and the
  market mechanism}, The Quarterly Journal of Economics, 84 (1970).

\bibitem{OddBall}
{\sc L.~Akoglu, M.~McGlohon, and C.~Faloutsos}, {\em oddball: Spotting
  anomalies in weighted graphs}, in PAKDD, M.~J. Zaki, J.~X. Yu, B.~Ravindran,
  and V.~Pudi, eds., 2010.

\bibitem{DBLP:journals/corr/AkogluTK14}
{\sc L.~Akoglu, H.~Tong, and D.~Koutra}, {\em Graph-based anomaly detection and
  description: {A} survey}, CoRR, abs/1404.4679 (2014).

\bibitem{barocas2014datas}
{\sc S.~Barocas and A.~D. Selbst}, {\em {Big Data's Disparate Impact}}, SSRN
  eLibrary,  (2014).

\bibitem{Chandola:2009:ADS:1541880.1541882}
{\sc V.~Chandola, A.~Banerjee, and V.~Kumar}, {\em Anomaly detection: A
  survey}, ACM Comput. Surv., 41 (2009).

\bibitem{fairness_nips17}
{\sc F.~Chierichetti, R.~Kumar, S.~Lattanzi, and S.~Vassilvitskii}, {\em Fair
  clustering through fairlets}, in NIPS, 2017.

\bibitem{sc_greedy}
{\sc V.~Chvatal}, {\em A greedy heuristic for the set-covering problem}, Math.
  Oper. Res., 4 (1979).

\bibitem{RIPPER}
{\sc W.~W. Cohen}, {\em Fast effective rule induction}, in ICML, 1995.

\bibitem{descriptive_clustering}
{\sc T.-B.-H. Dao, C.-T. Kuo, S.~S. Ravi, C.~Vrain, and I.~Davidson}, {\em
  Descriptive clustering: Ilp and cp formulations with applications}, in
  {IJCAI-18}, 2018.

\bibitem{DTDM}
{\sc I.~Davidson, A.~Gourru, and S.~Ravi}, {\em The cluster description problem
  - complexity results, formulations and approximations}, in NIPS, 2018.

\bibitem{heart_disease}
{\sc R.~Detrano, A.~J{\'a}nosi, W.~Steinbrunn, M.~Pfisterer, J.~Schmid,
  S.~Sandhu, K.~Guppy, S.~Lee, and V.~Froelicher}, {\em International
  application of a new probability algorithm for the diagnosis of coronary
  artery disease}, American Journal of Cardiology, 64 (1989).

\bibitem{sc_inapprox}
{\sc U.~Feige}, {\em A threshold of ln n for approximating set cover}, J. ACM,
  45 (1998).

\bibitem{lookout}
{\sc N.~Gupta, D.~Eswaran, N.~Shah, L.~Akoglu, and C.~Faloutsos}, {\em Beyond
  outlier detection: Lookout for pictorial explanation}, in ECML/PKDD, 2018.

\bibitem{census}
{\sc R.~Kohavi}, {\em Scaling up the accuracy of naive-bayes classifiers: A
  decision-tree hybrid}, in KDD, 1996.

\bibitem{aaai16_tom}
{\sc C.-T. Kuo and I.~Davidson}, {\em A framework for outlier description using
  constraint programming}, in AAAI, 2016.

\bibitem{BRL}
{\sc B.~Letham, C.~Rudin, T.~H.~McCormick, and D.~Madigan}, {\em Interpretable
  classifiers using rules and bayesian analysis: Building a better stroke
  prediction model}, The Annals of Applied Statistics, 9 (2015).

\bibitem{explain_by_subspace}
{\sc M.~Macha and L.~Akoglu}, {\em {X-PACS:} explaining anomalies by
  characterizing subspaces}, CoRR, abs/1708.05929 (2017).

\bibitem{pnpsc}
{\sc P.~Miettinen}, {\em On the positive--negative partial set cover problem},
  Inf. Process. Lett., 108 (2008).

\bibitem{rbsc_approx}
{\sc D.~Peleg}, {\em Approximation algorithms for the label-covermax and
  red-blue set cover problems}, in SWAT, 2000.

\bibitem{readmission}
{\sc B.~Strack, J.~Deshazo, C.~Gennings, J.~L. Olmo~Ortiz, S.~Ventura, K.~Cios,
  and J.~Clore}, {\em Impact of hba1c measurement on hospital readmission
  rates: Analysis of 70,000 clinical database patient records}, BioMed research
  international, 2014 (2014).

\bibitem{manufacturing_anomaly}
{\sc G.~A. Susto, M.~Terzi, and A.~Beghi}, {\em Anomaly detection approaches
  for semiconductor manufacturing}, Procedia Manufacturing, 11 (2017).

\bibitem{BRS}
{\sc T.~Wang, C.~Rudin, F.~Doshi-Velez, Y.~Liu, E.~Klampfl, and P.~MacNeille},
  {\em A bayesian framework for learning rule sets for interpretable
  classification}, Journal of Machine Learning Research, 18 (2017).

\end{thebibliography}

\end{document}